\documentclass{article}
\usepackage{ijcai21}

\usepackage{times}

\usepackage{soul}
\usepackage{url}
\usepackage[hidelinks]{hyperref}
\usepackage[utf8]{inputenc}
\usepackage[small]{caption}
\usepackage{graphicx}
\usepackage{amsmath}
\usepackage{amsthm}
\usepackage{booktabs}
\usepackage{changebar}
\usepackage{multirow}

\usepackage[linesnumbered,ruled,vlined]{algorithm2e}
\usepackage{tabularx}
\usepackage{xcolor}
\usepackage{makecell}
\usepackage{amssymb}

\usepackage{tikz}
\usetikzlibrary{arrows,positioning,chains,fit,shapes,calc}

\usepackage{color, colortbl}

\usepackage[inline]{enumitem}

\definecolor{Gray}{gray}{0.9}

\usepackage{mathabx}

\newcommand\cons[1]{\widehat{#1}}
\newcommand\pars[1]{\widecheck{#1}}

\newcommand{\var}{\textbf{X}}
\newcommand\inst[1]{\textbf{#1}}

\theoremstyle{definition}
\newtheorem{definition}{Definition}
\newtheorem{remark}{Remark}
\newtheorem{example}{Example}

\newtheorem{proposition}{Proposition}
\newtheorem{lemma}{Lemma}

\title{Sufficient reasons for classifier decisions in the presence of constraints}

\author{
Niku Gorji\footnote{Contact Author}\and
Sasha Rubin\\
 \affiliations
School of Computer Science\\
The University of Sydney, Australia
 \emails
\{niku.gorji, sasha.rubin\}@sydney.edu.au
}

\begin{document}

\maketitle

\begin{abstract}

    Recent work has unveiled a theory for reasoning about the decisions made by binary classifiers: a classifier describes a Boolean function, and the reasons behind an instance being classified as positive are the prime-implicants of the function that are satisfied by the instance. One drawback of these works is that they do not explicitly treat scenarios where the underlying data is known to be constrained, e.g., certain combinations of features may not exist, may not be observable, or may be required to be disregarded. We propose a more general theory, also based on prime-implicants, tailored to taking constraints into account. The main idea is to view classifiers in the presence of constraints as describing partial Boolean functions, i.e., that are undefined on instances that do not satisfy the constraints. We prove that this simple idea results in reasons that are no less (and sometimes more) succinct. That is, not taking constraints into account (e.g., ignored, or taken as negative instances) results in reasons that are subsumed by reasons that do take constraints into account. We illustrate this improved parsimony on synthetic classifiers and classifiers learned from real data.
\end{abstract}

\section{Introduction}

Machine learning models are currently used to assist decision makers. In case the domains involve high-stakes decisions (e.g., criminal justice and healthcare), model understanding can be used to help with debugging, detect bias in predictions, and vet models~\cite{explainml-tutorial}. One form of model-understanding is to explain decisions of pre-built models in a post-hoc manner. Models that may not lend themselves easily to interpretability (e.g., some neural networks, random forests or Bayesian network classifiers) or in cases there is an accuracy-interpretability tradeoff, can be passed as inputs to algorithms to produce explanations~\cite{explainml-tutorial}.

In this paper we follow a recent line of logic-based approaches for supplying explanations behind individual classifier decisions~\cite{shih2018symbolic,darwiche2020reasons} (rather than explaining the whole model). These works treat the input-output behaviour of a learned binary classifier as a total Boolean function, independently of how the classifier was learned or is implemented. In order to extract specific reasons, these works encode and manipulate the functions symbolically, e.g., as BDDs.

The problem we address in this work is how to incorporate background knowledge, specifically input/domain constraints, into supplying reasons behind individual decisions of classifiers.\footnote{We do not deal with constraints on the possible outputs of a classifier~\cite{xu2018semantic}.}
Such constraints may arise from  the structure and inter-dependencies between features present in data~\cite{darwiche2020three}. For example, in a medical setting, some combinations of drugs may never be prescribed together (and thus will not appear in any dataset or clinical setting): if we know that drug A and drug B are never prescribed together (this is the constraint), then a reason of the form ``drug A was prescribed and drug B was not prescribed" is overly redundant; it is more parsimonious to supply the reason ``drug A was prescribed". 

Our contributions are as follows:
\begin{enumerate}
\item 
We provide a crisp formalisation that takes constraints into account, resulting in reasons that are as least (and sometimes more) parsimonious, i.e., more general and more succinct, than not taking constraints into account. The idea is to present a classifier as a partial function that is not defined on those input instances that do not satisfy the constraint, and then to use the classic definition of prime-implicant on partial functions~\cite{coudert1994two} as the instantiation of the word ``reason''. This immediately and naturally generalises the work of \cite{shih2018symbolic,darwiche2020reasons} from the unconstrained setting to the constrained setting. 
\item  The main computational problem is to find all reasons of a classifier-decision (for a given instance) in the presence of constraints. We provide a simple reduction of this problem to the unconstrained setting, allowing one to {re-use} existing algorithms and tools. The idea is that if the constraint is given by the Boolean formula $\kappa$, and the decision function by $\varphi$, then reasons of $\varphi$-decisions that take $\kappa$ into account are exactly the reasons of $(\kappa \to \varphi)$-decision.\footnote{This simple reduction can be done in linear-time for formulas represented as parse-trees, and in polynomial-time for formulas represented as $\textrm{OBDD}_<$ \cite{darwiche2002knowledge}.} Interestingly, all other variations, (including the seemingly natural $(\kappa \wedge \varphi)$), provide no more, and sometimes less, parsimonious reasons.
\item We show, both theoretically and empirically on synthetic classifiers and classifiers learned from data, that approaches that ignore constraints may supply reasons that 
are unnecessarily long since they redundantly encode knowledge already described in the constraints. 
\end{enumerate}

\section{Preliminaries}

We begin by recalling just enough logical background to be able to explain our theory.

\paragraph{Boolean logic, partial Boolean functions, and prime-implicants.}  Let $\var =\{X_1, X_2, \cdots, X_n\}$ be a set of $n$ Boolean variables (aka \emph{feature variables}).  The set of \emph{Boolean formulas} is generated from $\var$, the constants $\top$ (verum/true) and $\bot$ (falsum/false), 
and the logical operations $\land,\lor,\lnot, \to$ and $\leftrightarrow$. Variables $X$ and their negations $\lnot X$ are called  \emph{literals}.

A \emph{term} ${t}$ is a conjunction of literals; the empty-conjunction is also denoted $\top$.

An \emph{instance (over $\var$)} is an element of $\{0,1\}^n$, and is denoted $\inst{x}$ (intuitively, it is an instantiation of the variables $\var$). An instance $\inst{x}$ \emph{satisfies} a formula $\varphi$ if $\varphi$ evaluates to true when the variables in $\varphi$ are assigned truth-values according to $\inst{x}$. The set of instances that satisfy the formula $\varphi$ is denoted $[\varphi]$. Thus we can represent sets of instances by formulas, i.e., the set $[\varphi]$ is represented by $\varphi$. If $[\varphi] = [\psi]$ then we say that $\varphi,\psi$ are \emph{logically equivalent}, i.e., they mean the same thing.

For terms $s,t$, we say that \emph{$s$ subsumes $t$} if $[{t}] \subseteq [{s}]$, i.e., if every instance that satisfies $t$ also satisfies $s$. If $[t] \subset [s]$ then we say that $s$ \emph{properly subsumes} $t$.

\begin{definition}
A \emph{partial Boolean function $f$ (over $\var$)} is a function $\{0,1\}^n \to \{0,1,*\}$. 
\end{definition}
For $i \in \{0,1,*\}$ define $f^i := f^{-1}(i)$. Call $f^1$ the function's \emph{onset}, $f^0$ its  \emph{offset}, and $f^*$ its \emph{don't-care set}. The \emph{care set} of $f$ is the set $f^0 \cup f^1$. If the don't-care set is empty, then $f$ is a \emph{total Boolean function}.
The instances in the onset are called \emph{positive instances of $f$}.
We can represent a total Boolean function by a formula $\varphi$ such that $[\varphi]=f^1$.

The following definition {generalises} the notion of implicant and prime-implicant from total Boolean functions (cf.~\cite{quine1952problem,shih2018symbolic,darwiche2020reasons}) to partial Boolean functions $f$ 
 (cf.~\cite{mccluskey1956minimization,coudert1994two}).
\begin{definition}
A term $t$ is an \emph{implicant of $f$} if $[t] \subseteq f^1 \cup f^*$; it is \emph{prime} if no other implicant of $f$ subsumes $t$.
\end{definition}

Intuitively, $t$ is prime if removing any literal from $t$ results in a term that is no longer an implicant.

\paragraph{Decision functions} \label{sec:DH2020}

Total Boolean functions naturally arise as the decision-functions of threshold-based binary classifiers~
\cite{choi2017compiling,shih2018symbolic}: the \emph{decision function} $f$ of a threshold-based classifier is the function that maps an instance $\textbf{x}$ to $1$ if $Pr(d=1 | \textbf{x}) \geq T$, and to $0$ otherwise; here $d$ is a binary class variable, and $Pr$ is the distribution specified by the classifier, {and $T$ is a user-defined classification threshold}. Note that decision functions thus defined are \textbf{total} Boolean functions.
Let $\inst{x}$ be a positive instance of $f$. In~\cite{darwiche2020reasons}, the \emph{sufficient reasons} behind the decision $f(\inst{x})=1$ are defined to be the terms $t$ such that (i) $t$ is a prime-implicant of $f$, and (ii) $t$ is satisfied by $\inst{x}$.

\begin{example} \label{ex:biconditional}
Consider the total Boolean function $f$ over $\var = \{X_1,X_2\}$ represented by the formula $(X_1 \leftrightarrow X_2)$ (see the third column of Table~\ref{tbl:PBF}).
The prime-implicants of $f$ are $(X_1 \land X_2)$ and $(\lnot X_1 \land \lnot X_2)$.

The only sufficient reason of $f(0,0)=1$ is the term $(\lnot X_1 \land \lnot X_2)$, and the only sufficient reason of $f(1,1)=1$ is the term $(X_1 \land X_2)$.
\end{example}

\section{Problem Setting } \label{sec:problem}

The problem we address is how to define reasons behind decision-functions in the presence of domain constraints.

\begin{definition}
A \emph{constraint} is a set $C$ of instances over $\var$. We can represent a constraint by a formula $\varphi$ such that $[\varphi]=C$.
\end{definition}

The following are just a few examples that show that constraints are ubiquitous. In a medical setting, constraints of the form $(X_1 \to X_2)$ may reflect that people with condition $X_1$ also have condition $X_2$, e.g., $X_1$ may mean ``is pregnant'' and $X_2$ may mean ``is a woman''. In a university degree structure, constraints of the form $X_1 \to (X_2 \land X_3)$ may reflect that $X_2$ and $X_3$ are prerequisites to $X_1$; constraints of the form $X_1 \to \lnot(X_2 \lor X_3)$ may reflect prohibitions; 
and constraints of the form $X_1 \land X_2$ may reflect compulsory courses. In configuration problems, e.g., that arise when users purchase products, the user may be required to configure their product subject to certain constraints, and constraints of the form $(X_1 \lor X_2) \land \lnot (X_1 \land X_2)$ may reflect that the user needs to choose between two basic models. These constraints also result from one-hot encodings of a categorical variables, e.g., if $M$ is a 12-valued variable representing months, and $X_i$ for $i = 1, \cdots, 12$ is Boolean variable, then the induced constraint is $\left(\bigvee_i X_i\right) \land \left(\bigwedge_{i \neq j} \lnot (X_i \land X_j)\right)$. Finally, combinatorial objects have natural constraints, e.g., rankings are ordered sets, trees are acyclic graphs, and games have rules, see Section~\ref{sec:experiments}.

Just as a binary classifier describes a total Boolean function, a binary classifier in the presence of a constraint $C$ describes a \emph{partial Boolean function $f$ whose care-set is $C$}. Indeed, the assumption in this paper is that constraints are \emph{hard}, i.e., instances  that are not in $C$ {are not possible and} can be ignored (e.g., they will not appear in training or testing data).

However, many techniques in Machine Learning, such as threshold-based classifiers,  produce representations of \emph{total} Boolean functions. This suggests the following terminology:

\begin{definition}
A \emph{constrained decision-function} is a pair $(f,C)$ consisting of a total Boolean function $f$ and a constraint $C$.
\end{definition}

We thus ask: 

\begin{center}
\begin{tabular}{ |p{78mm}| } 
 \hline

How should one define reasons behind decisions of constrained decision-functions?\\

 \hline
\end{tabular}
\end{center}

We posit that a reasonable notion of ``reason" that takes constraints into account should have the following properties:
\begin{enumerate}
\item it does not depend on the representation of $(f,C)$, i.e., it is a semantic notion; 
\item it does not depend on the values $f(x)$ for $\inst{x} \not \in C$, i.e., if $f,g$ agree on $C$ (and perhaps disagree on the complement of $C$), then reasons for $(f,C)$ should be the same as reasons for $(g,C)$; 

\item it is not less succinct (and is sometimes more succinct) than not taking constraints into account;
\item in case there are no constraints, i.e., $C = \{0,1\}^n$, we recover the notion of reasons from~\cite{darwiche2020reasons,shih2018symbolic}. 
\end{enumerate}
Our formalisation in the next section satisfies these properties.

\section{Formalisation of Reasons in the Presence of Constraints}

The main objective of this work  is to provide a principled definition for reasons behind decisions made by constrained decision-functions $(f,C)$. Let $f_C$ be the partial Boolean function that maps $\inst{x}$ to $f(\inst{x})$ if $\inst{x} \in C$, and to $*$ otherwise. Item 2 above says that we want our definition to only depend on the partial Boolean function $f_C$; so, we define reasons for partial Boolean functions~$g$.

\begin{definition}[Sufficient reasons of partial Boolean function] \label{dfn:sufficientreason} 
Let $g$ be a partial Boolean function and let $\inst{x}$ be a positive instance of $g$. A \emph{sufficient reason of $g(\inst{x})=1$} is a term $t$ such that (i) $t$ is a prime-implicant of $g$, and (ii) $t$ is satisfied by $\inst{x}$.
\end{definition}

Then: define the sufficient reasons of a constrained decision-function $(f,C)$ to be those of the induced partial Boolean function $f_C$, i.e., if $\inst{x} \in C$ is a positive instance of $f$, then a \emph{sufficient reason of $\inst{x}$ wrt $(f,C)$} is a term $t$ such that (i) $t$ is a prime-implicant of $f_C$, and (ii) $t$ is satisfied by $\inst{x}$.

\begin{example} \label{ex:CDF}
Continuing the Example, suppose a constraint is specified by the formula $(X_1 \to X_2)$, thus $C  = \{(0,0), (0,1), (1,1)\}$. Table~\ref{tbl:PBF} provides both $f$ and the partial Boolean function $f_C$.
\begin{table}[ht]
     \begin{center}
     
        \begin{tabular}[h]{cc||cc}
        \toprule
        $X_1$ & $X_2$ & $f$ & $f_C$ \\
        \hline
        0 & 0 & 1 & 1  \\
        0 & 1 & 0 & 0 \\
        \rowcolor{Gray} 1 & 0 & 0 & * \\       
        1 & 1 & 1 & 1  \\
        \bottomrule
        \end{tabular}

      \caption{Partial Boolean function determined by the total function $f$ and the constraint $C$ specified by $(X_1 \to X_2)$. The row corresponding to the instance not in the constraint is greyed-out.}
      \label{tbl:PBF}
      \end{center}
\end{table}
The prime-implicants of $f_C$ are $\lnot X_2$ and $X_1$. 
The only sufficient reason for $f_C(0,0)=1$ is $\lnot X_2$, and the only sufficient reason for $f_C(1,1)=1$ is $X_1$.
\end{example}

\begin{remark}
Sufficient reasons of a negative instances can be handled by sufficient reasons of positive instances of the dual function $g$ defined as follows: $g(\inst{x}) = 0$ if $f(\inst{x}) = 1$; $g(\inst{x})=1$ if $f(\inst{x}) = 0$; and $g(\inst{x}) = *$ otherwise.
\end{remark}

We now justify Definition~\ref{dfn:sufficientreason} (Sufficient Reasons).

\begin{enumerate}

\item Our definition is a {generalisation} of  \emph{sufficient reason} from \cite{darwiche2020reasons}, called \emph{PI-explanation}  in 
\cite{shih2018symbolic}. Indeed, those works only handle decision-functions without constraints (i.e., in our terminology, those works have $f^* = \emptyset$). This explains why we use the same terminology, i.e., \emph{sufficient reasons}.

\item Using prime implicants to explain Boolean functions as well as individual instances has a long history in Science (see Section~\ref{sec:relatedwork} for a full discussion). This is generally motivated by the principle of parsimony, also known as Occam's razor in scientific domains, i.e., a reason should not be any more complicated than it needs to be. 

\item Subtle changes in the definition result in radically different types of reasons. This is discussed at the end of Section~\ref{sec:useful properties}.
\end{enumerate}

Assuming the constraints are known, the second justification means that \emph{reasons should not contain redundancies that are captured by the constraint}.
We illustrate this point by showing that, given a constrained decision function $(f,C)$, simply considering reasons of the total Boolean function $f$ (and ignoring the constraint $C$), may supply less parsimonious reasons.

\begin{example} \label{ex:compare}
Continuing the Example, note that the only sufficient reason for $f(0,0)=1$ is $(\lnot X_1 \land \lnot X_2)$ which is subsumed by $\lnot X_2$, a sufficient reason of $f_C(0,0)=1$; similarly, the only sufficient reason for $f(1,1)=1$ is $(X_1 \land X_2)$ which is subsumed by $X_1$, a sufficient reason of $f_C(1,1)=1$.
\end{example}

Example~\ref{ex:compare} accords with the intuition that, in light of the constraint $(X_1 \to X_2)$,  reason $X_1$ is preferred to reason $(X_1 \land X_2)$. This is no accident: reasons using $f_C$ are as least as parsimonious as reasons using $f$, as we know prove.

\begin{lemma} \label{lem:subsumed}
If $f,g$ are partial functions such that $f^1 \cup f^* \subseteq g^1 \cup g^*$, then every sufficient reason for $f(\inst{x})=1$ is subsumed by some sufficient reason of $g(\inst{x})=1$.
\end{lemma}

\begin{proof}
Note that a sufficient reason of $f(\inst{x})=1$, being an implicant of $f$, is also an implicant of $g$ (since $f^1 \cup f^* \subseteq g^1 \cup g^*$), and that every implicant of a function is subsumed by some prime-implicant of that function.
\end{proof}

The following proposition establishes that $f_C$ supplies reasons that are as least as parsimonious as reasons using $f$:

\begin{proposition} \label{thm:subsumed}
Let $(f,C)$ be a constrained decision-function, and suppose $\inst{x} \in C$ is a positive instance of $f$. Then every sufficient reason of $f(\inst{x})=1$ is subsumed by some sufficient reason of $f_C(\inst{x})=1$.
\end{proposition}

\begin{proof}
Lemma~\ref{lem:subsumed} applies since $f^1 \cup f^* = f^1 \subseteq (f_C)^1 \cup (f_C)^*$. Indeed, the equality holds since $f$ is assumed total, and the containment holds by definition of $f_C$.
\end{proof}

It is not hard to find examples where the subsumption in Proposition~\ref{thm:subsumed} is always strict, see Example~\ref{ex:compare}. In particular, if $t$ properly subsumes $s$ then, $t$ is smaller than $s$ (i.e., contains less literals). And indeed, it is not hard to find examples where every sufficient reason of $f(\inst{x})=1$ is much larger than every sufficient reason of $f_C(\inst{x})=1$. To do this, consider the constraint $X_1 \to (X_2 \land X_3 \cdots \land X_n)$: then every reason of the form $X_1 \land X_2 \land \cdots \land X_n$ is subsumed by the reason $X_1$. %

\subsection{Reducing the  Constrained Case to the Unconstraint Case} \label{sec:useful properties}

It turns out that that there is a particular \emph{total} function whose reasons are exactly the same as those of a given partial Boolean function. 
This will allow us to {reuse} algorithms and tools that are already developed for reasoning about total Boolean functions, e.g.,~\cite{shih2018symbolic,darwiche2020reasons}.

\begin{definition} \label{dfn:fbar}
For a partial Boolean function $f$, define the total Boolean function $\pars{f}$ as follows: the value of $\pars{f}$ on $\inst{x}$ is equal to $1$ if $f(\inst{x}) = 1$ or $f(\inst{x}) = *$, and equal to $0$ otherwise.

\end{definition}

Sufficient reasons using $f$ and $\pars{f}$ are the same:
\begin{proposition} \label{thm:fbar}
Let $f$ be a partial Boolean function and $\inst{x}$ a positive instance of $f$. A term is a sufficient reason of $f(\inst{x})=1$ iff it is a sufficient reason of $\pars{f}(\inst{x})=1$. 
\end{proposition}

\begin{proof} 
Follows from the definitions, i.e., the implicants of $f$ and $\pars{f}$ are the same (since $f^1 \cup f^* = \pars{f}^1$), and thus the prime implicants of $f$ and $\pars{f}$ are the same, and a positive instance of $f$ is also a positive instance of $\pars{f}$.
\end{proof}

Thus, for a constrained decision-function $(f,C)$ with $f$ represented by the formula $\varphi$, and $C$ by the formula $\kappa$, the total function $\pars{f}$ is represented by $\kappa \to \varphi$. We are now in a position to illustrate the third point in the justification of Definition~\ref{dfn:sufficientreason}, i.e., that subtle changes in the definition result in radically different types of reasons. First, we have seen in the Examples that ignoring the constraints does not provide the most parsimonious reasons. Second, consider the variation in which one uses the total function represented by the formula $\kappa \land \varphi$. Although seemingly natural (indeed, why not assign a negative value to instances that do not satisfy the constraint), it is not hard to see, using Lemma~\ref{lem:subsumed}, that this results in the \emph{least parsimonious} reasons for $f$. On the other hand, using the formula $\kappa \to \varphi$ (or equivalently $f_C$ or equivalently $\pars{f}$), results in the \emph{most parsimonious} reasons. We find it striking that this change of perspective (implication vs conjunction) has such drastic changes on the parsimony of the produced reasons.

\subsection{Equivalent reasons}

We now explain that syntactically different reasons may be semantically equivalent, and thus we only need to consider one representative from each equivalence-class.

Suppose $t,s$ are syntactically different terms, but logically equivalent, i.e., $[t] = [s]$. Then there is no semantic difference between the terms, and so we do not distinguish between them. E.g., $t = X_1 \land X_2$ and $s = X_2 \land X_1$ are logically equivalent, and thus we generally do not distinguish between $s,t$. 
However, even if $s,t$ are not logically equivalent, they may be \emph{logically equivalent modulo $C$}, i.e., $C \cap [s] = C \cap [t]$. In this case, we say that $s,t$ are or simply \emph{constraint-equivalent}. For instance, if $C$ is represented by $(X_1 \lor X_2) \land \lnot (X_1 \land X_2)$ then $t = \lnot X_1$ is $C$-equivalent to $s = X_2$. Thus, we do not distinguish between $s$ and $t$. Practically speaking, we will only consider reasons up to constraint-equivalence, i.e., we will pick an arbitrary representative from each constraint-equivalence class.

More generally, say that $s$ is \emph{constraint-subsumed} by $t$ if $[s] \cap C \subseteq [t] \cap C$. Just as in the unconstrained case sufficient reasons are maximal in the subsumption order, it may be reasonable in the constrained case to require that reasons are maximal in the constraint-subsumption order. We do not require this in our notion of sufficient reason since, doing it may eliminate more succinct reasons, which might be undesirable. These tensions are illustrated in the Case Studies in Section~\ref{sec:experiments}, and further discussed in Section~\ref{sec:relatedwork}.

\section{Illustration} \label{sec:illustration}

In this section we illustrate sufficient reasons on a complete synthetic example of a learned classifier, inspired by an example in \cite{kisa2014probabilistic}.

Consider a tech-company that is trying to decide whether or not to shortlist a recent graduate from a particular school of computer science for a job interview. The company considers students who took any of Probability (P), Logic (L), Artificial Intelligence (A) or Knowledge Representation (K). Suppose that the company uses data on students who were hired in the past to learn a threshold-based classifier, and let $f$ be the associated \emph{total} decision function over $\var = \{L,K,P,A\}$ with onset $f^1 = \{(0011),(0110),(0111),(1100),(1101),(1110),(1111)\}$.

The details of how such an $f$ can be learned is not the focus of this paper, see, e.g., \cite{shih2018symbolic} for the case of Bayesian networks, \cite{choi2017compiling} for the case of Neural Networks or to \cite{audemard2020tractable} for Random Forests.

Consider an instance $\inst{x} = (0011)$ corresponding to students that did not take L or K, but did take P and A. Note that $f(\inst{x})=1$, i.e., the classifier decides to grant such students interviews. What is the reason behind this decision? Table~\ref{tbl:reasons} gives the reasons, and we see that the only reason behind the decision of $f$ for $\inst{x} = (0011)$ is $(\lnot L \land P \land A)$, i.e., that the student did not take L, but did take P and A.

\begin{table}[ht]

        \begin{tabular}[t]{p{1mm}p{1mm}p{1mm}p{1mm}|p{3.2cm}|p{2cm} } 
          \toprule

         L & K & P & A  & Reasons for $f(\inst{x})=1$ &  Reasons for $f_C(\inst{x})=1$ \\
        \hline
        0 & 0 & 1 & 1  & 
            $(\neg L \land P \land A)$  & 
            $(\neg L \land A)$ \\        
        0 & 1 & 1 & 1 & 
            $(\neg L \land P \land A)$, $(K \land P)$ & 
             $(\neg L \land A)$, $K$ 
            \\      
        1 & 1 & 0 & 0 & 
            $(L \land K)$ & 
            $K$ 
            \\     
        1 & 1 & 1 & 0 & 
            $(L \land K)$, $(K \land P)$ & 
            $K$ 
            \\    
        1 & 1 & 1 & 1 & 
            $(L \land K)$, $(K \land P)$ & 
            $K$ \\

        \bottomrule
      \end{tabular}
      \caption{Positive instances that satisfy the constraints, and reasons.
      }
      \label{tbl:reasons}

\end{table}

Suppose, that a student's enrolments must satisfy the following constraints $C$: a student must take P or L,  
$(P \lor L)$;
  the prerequisite for A is P,
  $(A \rightarrow  P)$;
 the prerequisite for K is A or L, 
  $(K \rightarrow  (A \lor L))$. Reasons of the constrained decision-function $f_C$ are given in Table~\ref{tbl:reasons}. Note $(\lnot L \land A)$ and $K$ are not constraint-equivalent, and indeed are incomparable in the constraint-subsumption order.

Consider the reason behind the decision $f(\inst{x})=1$ where $\inst{x} = (0011)$, i.e., $\lnot L \land A$.  This reason strictly subsumes the reason $\lnot L \land P \land A$ used by the original (unconstrained) classifier $f$. This phenomenon, that for every positive instance $\inst{x}$ in $C$, every sufficient reason of $f(\inst{x})=1$ is subsumed by some sufficient reason of $f_C(\inst{x})=1$, can be seen in all other rows of Table~\ref{tbl:reasons}. This illustrates that our notion  of sufficient reason (Definition~\ref{dfn:sufficientreason}) systematically eliminates such redundancies, a fact we  formalised in Proposition~\ref{thm:subsumed}.

\section{Case Studies and Validation} \label{sec:experiments}

In this section we validate our theory on constrained decision-functions learned from data.\footnote{For the experiments, we restrict our attention to binary data. Continuous data can be discretised, and discrete/categorical data can be binarised as shown in~\cite{breiman1984classification}.}
Figure~\ref{fig:schematic} illustrates the workflow. 
The case-studies showcase two major types of constraints that
can arise in AI: (i) constraints due to pre-processing of data
(e.g. one-hot, or other categorical, encodings); (ii) inherent semantic-constraints in input space.

We focus on interpretable classifiers, namely decision trees (learned using the Recursive Partitioning RPART algorithm in R \cite{therneau1997introduction}). This allows us to compare sufficient reasons (both with and without constraints) of decisions with the reasons specified by the corresponding branch of the decision tree~\cite{breiman1984classification}

\paragraph{Algorithms} \label{sec:algorithms}
We consider the following computational problem: given a partial Boolean function $f$, and a positive instance $\inst{x}$, find all sufficient reasons of $f(\inst{x})=1$. By Proposition~\ref{thm:fbar}, one can reuse algorithms developed for total Boolean functions: simply replace $f$ by the total Boolean function $\pars{f}$.

For instance, the computational problem is solved in \cite{shih2018symbolic} for total Boolean functions by computing prime-implicants using the Shannon-expansion recursive procedure~\cite{coudert1994two} and limiting the recursive calls at every step to those that satisfy a given instance $\inst{x}$.

Moreover,~\cite{shih2018symbolic} use circuit-representations of functions and sets of prime-implicants. This is the approach we use in our experiments. 

In particular, for a given variable order $v$, we build a few $\textrm{OBDD}_<$ \cite{darwiche2002knowledge} with variable order $v$. First, an $\textrm{OBDD}_<$ representing $C$ which we complement, and then disjoin that complement with an $\textrm{OBDD}_<$ representing the decision-function $f$, in order to get the $\textrm{OBDD}_<$ representing the function $\pars{f_C}$.

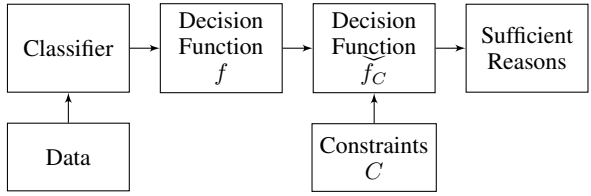
\begin {figure}[t]
\centering
\begin{tikzpicture}[node distance=4mm, >=latex',
 block/.style = {draw, rectangle, minimum height=5mm, minimum width=18mm,align=center},
rblock/.style = {draw, rectangle, rounded corners=0.5em},
tblock/.style = {draw, trapezium, minimum height=10mm, 
                 trapezium left angle=75, trapezium right angle=105, align=center},
scale=1, every node/.style={scale=0.9}
                        ]
    
    \node [block, minimum height=13mm]       (classifier)   {Classifier};
    \node [block, right=of classifier, minimum height=13mm]     (bf)  {Decision\\Function\\$f$};
    \node [block, right=of bf, minimum height=13mm]     (formula)  {Decision\\Function\\$\pars{f_C}$};
    \node [block, right=of formula, minimum height=13mm]    (reasons)      {Sufficient\\ Reasons};
    \node [block, below=of classifier, minimum height=10mm]                      (data)     {Data};
    \node [block, below=of formula, minimum height=10mm]     (constraints1)  {Constraints\\$C$};

    \path[draw,->] (data)      edge (classifier)
                    (classifier)    edge (bf)
                   (bf)   edge (formula)
                   (formula)   edge (reasons)
                   ;

     \path[draw,<-] 
                    (formula)       edge [midway](constraints1)
                    ;
\end{tikzpicture}

\caption{Workflow}
\label{fig:schematic}
\end{figure}

\paragraph{Case Study 1.}  We used the dataset of Corticosteroid Randomization after Significant Head Injury (CRASH) trial \cite{mrc2008predicting} and based our study on 11 clinically relevant variables described in \cite{zador2016predictors}.
Input variables include demographics, injury characteristics and image findings,
six of which are categorical, and the rest are Boolean.\footnote{Categorical variables are: \textbf{A}ge(7), \textbf{E}ye(4) \textbf{M}otor(6), \textbf{V}erbal(5), \textbf{P}upils(3). The Boolean are: $EC, PH, OB, SA, MD, HM$.}

Outcome variable $gos$ $1$ ($0$) indicates favourable (unfavourable) outcomes (e.g., death or severe disability).

Categorical variables are encoded using a one-hot encoding, which induces the constraint $C$ as follows. For a categorical variable $X$, let $D$ denote a set of Boolean variables corresponding to the set of categories of $X$. The corresponding constraint says that exactly one of the variables in $D$ must be true. 
For example, variable $Eye$ (shortened to $E$) has 4 categories, which we encode by the Boolean variables in $D_{E} = \{E_1, E_2, E_3, E_4 \}$. The corresponding constraint is $\bigvee_i E_i \land \bigwedge_{i \neq j} \lnot (E_i \land E_j)$, where $i,j$ vary over $\{1,2,3,4\}$. The constraint $C$ is the conjunction of all such constraints, one for each categorical variable.

Following \cite{steyerberg2008predicting} we base our example on 6945 cases with no missing values. 
RPART (seed: 25, train: 0.75, cp: 0.005) correctly classifies 75.69\% of instances in the test set (ROC 0.77).

Figure \ref{fig:rpart} illustrates the model.

\begin{figure}[h]
  \centering
      \includegraphics[width=140pt]{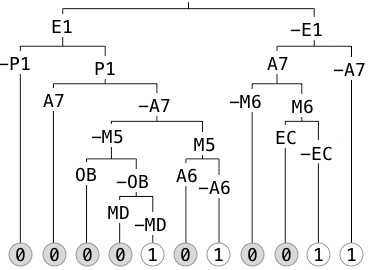}
        \caption{RPART decision tree for case study1.}
        \label{fig:rpart}
\end{figure}

Consider the instance $\inst{x}$ that maps 
 $A_1$, $E_1$, $M_5$, $V_2$, $P_1$, $OB$, $MD$ to $1$, and all the other variables to $0$.

There is one sufficient reason obtained from the decision-function $f$ of the decision tree:

            $\lnot A_6 \land \lnot A_7 \land M_5 \land P_1$.

There are eight sufficient reasons obtained from $\pars{f_C}$ (up to logical equivalence), but only two up to constraint-equivalence: Reason 1) $A_1 \land M_5 \land P_1$ and Reason 2) $\lnot A_6 \land \lnot A_7 \land M_5 \land P_1$.

\paragraph{Discussion of Case-Study 1.}

The standard explanation from the learned decision tree is $E_1 \land P_1   \land \lnot A_7 \land M_5 \land \lnot A_6$. It is strictly subsumed by (and thus longer than) the sufficient reason using $f$. This shows that decision-rules may not be the best explanations. Further, as we see, taking the constraints into account may result in more succinctness. Note that, as guaranteed by Proposition~\ref{thm:subsumed}, the reason using $f$ is subsumed by some reason using $\pars{f_C}$, in fact it appears as the second reason.

We remark that reasons 1) and 2) are not constraint-equivalent (and thus should be considered different reasons). Which reason should one prefer? On the one hand, Reason 1) is more succinct. On the other hand, Reason 2) strictly constraint-subsumes Reason 1), i.e., it applies to more instances. Without additional preferences regarding succinctness versus generality, there is no reason to prefer one over the other, and thus we return both of them.

Finally, we remark that if one had used the function $\cons{f_C}$ instead, one would get one sufficient reason for this decision that is highly redundant in light of the one-hot constraint, i.e.,
$(A_1 \land E_1 \land M_5 \land V_2 \land P_1 \land \bigwedge_{X \in V} \lnot X)$
where $V$ consists of all the remaining variables $A_2, A_3, \cdots, E_2, E_3, \cdots$.

\paragraph{Case Study 2.}

We used the Tic-Tac-Toe Endgame dataset from the UCI machine learning repository as binarised in~\cite{verwer2019learning}: for each of the 9 board positions (labelled as in Table~\ref{tbl:ttt2}a.), introduce variables $F_{i,X}$ (resp. $F_{i,O}$) capturing whether or not $X$ (resp. $O$) was placed in that position. This induces constraints $C$ that are different to Case Study $1$. We let $C$ be the constraint that expresses that no position contains both an $X$ and a $O$ (although it may have neither), i.e., $C$ is represented by $\bigwedge_{0 \leq i \leq 8} \lnot (F_{i,X} \land F_{i,O})$. Outcome variable $won$ 1 (0) is a win (loss) for player X.

 \begin{figure}[h]
  \centering
      \includegraphics[width=\linewidth]{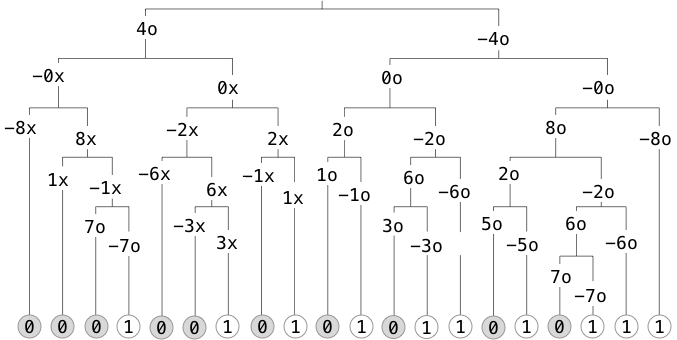}
        \caption{RPART decision tree for case study 2. We drop $F$ and write, e.g., $4o$ instead of $F_{4,O}$ for readability.}
        \label{fig:CRASH-DT}
\end{figure}

RPART (seed 1, train: 0.7, cp 0.01) gives 20 rules (Figure~\ref{fig:CRASH-DT}) with 93\% accuracy for the test set (ROC 0.97).

Consider the instance drawn in Table~\ref{tbl:ttt2}b.
\renewcommand{\arraystretch}{1.2}%
\begin{table}[h!]
  \centering
  \begin{tabular}{ccc}

    a.
    \begin{minipage} {.12\textwidth}
        \begin{center}
            \begin{tabular}[h]{c|c|c}
    
            0 & 1 &  2 \\
            \hline
            3  & 4 & 5 \\
            \hline
            6 & 7 & 8 \\       
            \end{tabular}

          \label{tbl:ttta}
        \end{center}
    \end{minipage}
    &
    b.
    \begin{minipage} {.12\textwidth}

         \begin{center}
            \begin{tabular}[h]{c|c|c}
    
            X & X & X  \\
            \hline
              &  &  \\
            \hline
            O  &  & O \\       
            \end{tabular}

          \label{tbl:tttb}
        \end{center}
    \end{minipage}
    &
    c.
    \begin{minipage} {.12\textwidth}

         \begin{center}
            \begin{tabular}[h]{c|c|c}
    
            01 & 01 & 01  \\
            \hline
            00  & 00 & 00 \\
            \hline
            10 & 00 & 10 \\       
            \end{tabular}

          \label{tbl:tttc}
        \end{center}
    \end{minipage}
  \end{tabular}
  \caption{a. TTT board; b. Positive instance; c. Encoded instance (cell $i$ is labelled by the values of $F_{i,X} F_{i,O})$.}
  \label{tbl:ttt2}
\end{table}

There are 8 sufficient reasons using $f$: 
For instance, Reason 1 is $\bigwedge_{i \in \{0,2,4,7\}} \lnot F_{i,O}$ and Reason 5 is $\bigwedge_{i \in \{0,2,7\}} \lnot F_{i,O} \land \bigwedge_{i \in \{0,1,2\}} F_{i,X}$. Note that the latter reason is redundant in light of the constraint $C$ (as witnessed, e.g., by the inclusion of the  literals $\lnot F_{0,O}$ and $ F_{0,X}$).
    \begin{table}[h!]
          \begin{tabular}{cccc}
        
                1.
            \begin{minipage} {.08\textwidth}
                \begin{center}
                0\texttt{-} \texttt{-}\texttt{-} 0\texttt{-} \\
                \texttt{-}\texttt{-} 0\texttt{-} \texttt{-}\texttt{-} \\
                \texttt{-}\texttt{-} 0\texttt{-} \texttt{-}\texttt{-}\\
                \end{center}
            \end{minipage}
            &
                    2.
            \begin{minipage} {.08\textwidth}
                    
                \begin{center}
                 0\texttt{-} \texttt{-}\texttt{-} \texttt{-}\texttt{-} \\
                 \texttt{-}\texttt{-} 0\texttt{-} 0\texttt{-} \\
                 \texttt{-}\texttt{-} 0\texttt{-} \texttt{-}\texttt{-}\\
                \end{center}
            \end{minipage}
            
             &
                     3.
            \begin{minipage} {.08\textwidth}
                \begin{center}
                    \texttt{-}\texttt{-} \texttt{-}\texttt{-} 0\texttt{-} \\
                    0\texttt{-} 0\texttt{-} \texttt{-}\texttt{-} \\
                    \texttt{-}\texttt{-} 0\texttt{-} \texttt{-}\texttt{-}\\
                \end{center}
            \end{minipage}
            &
          
                     4.
            \begin{minipage} {.08\textwidth}
                \begin{center}
                    \texttt{-}\texttt{-} 0\texttt{-} \texttt{-}\texttt{-} \\
                    0\texttt{-} 0\texttt{-} 0\texttt{-} \\
                    \texttt{-}\texttt{-} 0\texttt{-} \texttt{-}\texttt{-}\\
                \end{center}
            \end{minipage}\\
            ~\\
            
                5.
            \begin{minipage} {.08\textwidth}
                \begin{center}
                01 \texttt{-}1 01 \\
                \texttt{-}\texttt{-} \texttt{-}\texttt{-} \texttt{-}\texttt{-} \\
                \texttt{-}\texttt{-} 0\texttt{-} \texttt{-}\texttt{-}\\
                \end{center}
            \end{minipage}
            &
                    6.
            \begin{minipage} {.08\textwidth}
                    
                \begin{center}
                 
                01 \texttt{-}1 \texttt{-}1 \\
                \texttt{-}\texttt{-} \texttt{-}\texttt{-} 0\texttt{-} \\
                \texttt{-}\texttt{-} 0\texttt{-} \texttt{-}\texttt{-}\\

                \end{center}
            \end{minipage}
            
             &
                     7.
            \begin{minipage} {.08\textwidth}
                \begin{center}
                   \texttt{-}1 \texttt{-}1 01 0\texttt{-} \texttt{-}\texttt{-} \texttt{-}\texttt{-} \texttt{-}\texttt{-} 0\texttt{-} \texttt{-}\texttt{-}\\

                \end{center}
            \end{minipage}
            &
          
                     8.
            \begin{minipage} {.08\textwidth}
                \begin{center}
                    \texttt{-}1 01 \texttt{-}1 0\texttt{-} \texttt{-}\texttt{-} 0\texttt{-} \texttt{-}\texttt{-} 0\texttt{-} \texttt{-}\texttt{-}
                \end{center}
            \end{minipage}
            
          \end{tabular}
    \end{table}

For $\pars{f_C}$, there are 11 reasons, including reasons 1-4 of $f$. We show 4 due to space limitation. Reason A subsumes the rest of reasons 5-8 of $f$ and might be preferred.

    \begin{table}[h!]
          \begin{tabular}{cccc}
        
                \textbf{A.}
            \begin{minipage} {.08\textwidth}
                \begin{center}
                 \texttt{-}1 \texttt{-}1 \texttt{-}1 \\
                 \texttt{-}\texttt{-} \texttt{-}\texttt{-} \texttt{-}\texttt{-} \\
                 \texttt{-}\texttt{-} 0\texttt{-} \texttt{-}\texttt{-}
                \end{center}
            \end{minipage}

            & B.
            \begin{minipage} {.08\textwidth}
                \begin{center}
                \texttt{-}1 \texttt{-}\texttt{-} \texttt{-}\texttt{-} \\
                \texttt{-}\texttt{-} 0\texttt{-} 0\texttt{-} \\
                \texttt{-}\texttt{-} 0\texttt{-} \texttt{-}\texttt{-}
                \end{center}
            \end{minipage}
            &
               C.
            \begin{minipage} {.08\textwidth}
                \begin{center}
                 \texttt{-}\texttt{-} \texttt{-}\texttt{-} \texttt{-}1 \\
                 0\texttt{-} 0\texttt{-} \texttt{-}\texttt{-} \\
                 \texttt{-}\texttt{-} 0\texttt{-} \texttt{-}\texttt{-}
                \end{center}
            \end{minipage}

            & D.
            \begin{minipage} {.08\textwidth}
                \begin{center}
                0\texttt{-} \texttt{-}\texttt{-} \texttt{-}1 \\
                \texttt{-}\texttt{-} 0\texttt{-} \texttt{-}\texttt{-} \\
                \texttt{-}\texttt{-} 0\texttt{-} \texttt{-}\texttt{-}
                \end{center}
            \end{minipage}
     \end{tabular}
    \end{table}

Consider the following constraint $C'$ that captures that X moves first and players alternate moves:
$\bigvee_{S,T} (\psi_S \land \varphi_T)$ where $S,T$ vary over all subsets $S,T$ of $\{0,1,2,\cdots,8\}$  such that $S \cap T = \emptyset$, and $0 \leq |S| - |T| \leq 1$, and $\psi_S$ is 
$\Big(\bigwedge_{N \in S} F_{N,X} \Big) \land \big(\bigwedge_{N \in U \setminus S} \lnot F_{N,X} \big)$ and 
$\varphi_T$ is $\Big(\bigwedge_{N \in T} F_{N,O}\Big) \land \big(\bigwedge_{N \in U \setminus T} \lnot F_{N,O} \big)$.  The formula expresses that the set $S$ of positions where X has played is disjoint from the set $T$ where $O$ has played, and that either there are the same number of moves, or $X$ has played one more.
In this case there are 46 sufficient reasons for the instance above, 
none of which are constraint-$C'$ equivalent, including the following which are not subsumed by any of the reasons using just the binarisation constraint $C$.

 \begin{table}[h!]
          \begin{tabular}{cccc}
        i.
             \begin{minipage} {.08\textwidth}
                 \begin{center}
                    \texttt{-}\texttt{-} \texttt{-}\texttt{-} \texttt{-}\texttt{-} \\
                      00 00 00 \\
                     \texttt{-}0 00 \texttt{-}\texttt{-}
                
                     \end{center}
                 \end{minipage}
             &
    ii.
            \begin{minipage} {.08\textwidth}
                \begin{center}
                \texttt{-}\texttt{-} \texttt{-}\texttt{-} \texttt{-}\texttt{-}\\
                \texttt{-}\texttt{-} 00 00\\
                \texttt{-}\texttt{-} 00 1\texttt{-}
                \end{center}
            \end{minipage}
            
            &
    iii.
            \begin{minipage} {.08\textwidth}
                    
                \begin{center}
                \texttt{-}1 \texttt{-}\texttt{-} \texttt{-}\texttt{-}\\
                \texttt{-}0 00 \texttt{-}\texttt{-}\\
                1\texttt{-} 00 \texttt{-}\texttt{-}\\
                                
                \end{center}
            \end{minipage}
          
            &
         iv.   
            \begin{minipage} {.08\textwidth}
                \begin{center}
                 \texttt{-}\texttt{-} \texttt{-}\texttt{-} \texttt{-}\texttt{-} \\
                 00 00 0\texttt{-} \\
                 1\texttt{-} 0\texttt{-} \texttt{-}0
                \end{center}
            \end{minipage}
    \end{tabular}
            
\end{table}

For instance, Reason ii says that, \emph{given that we know the board is the result of a valid play}, if positions 4,5,7 are blank and position 8 has an O, then player X must have won. This is indeed correct: player O could not have won (since with 5 moves in the game player O can only move twice), and there could not be a draw (because not all positions were filled yet).

\section{Related Work and Discussion} \label{sec:relatedwork}

Our theory generalises \cite{shih2018symbolic,darwiche2020reasons} to handle domain constraints. We also show how to reduce the constrained-case to the unconstrained case, thus allowing one to reuse existing symbolic algorithms and tools~\cite{shih2018symbolic}. There are other approaches to handle the unconstrained case: purely heuristic (which do not provide guarantees on the quality of explanations) ~\cite{ribeiro2016should,lakkaraju2019faithful,iyer2018transparency} and a combination of heuristic and abductive reasoning~ \cite{ignatiev2019abduction}. 

\paragraph{Prime-implicants as Explanations.}
Prime implicants have been used for reasoning and providing explanations in a number of different settings: in Electrical Engineering for circuit minimisation~\cite{brayton1989exact,mcgeer1993espresso};
in model-based diagnosis in AI, as a set of diagnoses sufficient to explain every state of a system in which faulty behaviour is observed~\cite{de1992characterizing,reiter1987theory};  in system reliability analysis, as smallest combinations of events that could lead a system to failure~\cite{coudert1993fault}; in Bioinformatics as ``predictive explanations" of gene interaction networks \cite{yordanov2016method}; and in the Social Sciences, as minimal combinations of causal conditions~\cite{thiem2013boolean}. 

Prime implicants of partial Boolean functions \cite{coudert1994two} are based on the early works of \cite{nelson1955weak,mccluskey1956minimization} and were used to further minimize a Boolean function when the output of the function for some instance was considered to be inconsequential/undefined. Those instances are often assigned to an arbitrary output if doing so results in obtaining prime implicants with fewest numbers of literals \cite{mccluskey1956minimization}. Inspired by these earlier works, our definition of sufficient reasons for constrained decision functions, treats constraints in a systematic manner to produce more parsimonious explanations.

\paragraph{Explanations in ML literature.}
The ML literature has techniques for producing (post-hoc local) rule-based explanations which are similar in spirit to the logic-based method of this paper. 
Notably, the \emph{anchors} of~\cite{Ribeiro:AAAI18} are analogous to implicants. That work: 1) is probabilistic, e.g., it works directly on a probabilistic model while our method works on Boolean functions representing a possibly probabilistic model; 2) aims to optimise the coverage of anchors (i.e., the probability that the anchor applies to a random instance), which is {not} an analogue of \emph{prime-implicant}, and thus potentially misses out valid explanations (indeed, that work has no analogue of subsumption); 3) does not explicitly handle constraints, while this is the main focus of our work.

\paragraph{Discussion}
The crux of our paper shows how to handle constraints, and that ignoring constraints could result in unnecessarily long/complex reasons, as well as some  reasons that may be missed altogether.

A general critique of the prime-implicant based approach to reasoning is that reasons may become too large to comprehend when the number of variables is large. Notice that our method is a step towards improving this problem in the presence of constraints. If the shortest reason in presence of constraints is still too large to comprehend, not taking constraints into account may results in reasons that are even larger and even harder to comprehend.

To validate the claim that using constraints results in no longer, and sometimes shorter reasons, in the two case studies, we compared (with equivalence, subsumption, constraint-equivalence and constraint-subsumption tests), every single reason obtained from decision function $f$ with that of $\pars{f}$ and observed that while 
adding constraints may decrease or increase the \emph{number} of reasons, it never increases the \emph{size} of the shortest reasons (guaranteed by Proposition 1). 

In the illustrative example of Section 5, adding constraints reduced the number, as well as the size of reasons for some instances, while in both of the case studies of Section 6, adding constraints increased the number of reasons, but reduced the size of some reasons.
Furthermore, in Case Study 2 we demonstrated that when constraints are not taken into account, some reasons may be missed altogether, and provided some examples of such reasons.

In cases of multiple (constraint-inequivalent) reasons for a decision (even amongst the shortest ones), we do not supply a way to pick one reason over another, a challenging problem~\cite{lakkaraju2019faithful}. Indeed, preferring one reason over another would require \emph{additional assumptions} about preferred reasons, e.g., favouring succinctness over generality  \cite{miller1}. 

Another noteworthy point is that in both of the case studies, we used decision trees. 
In effect, what matters in our investigation, is not the \emph{type} of a classifier, nor the \emph{method} of obtaining its decision-function.

What matters is \emph{how} constraints are handled at the level of the decision function (not its representation), i.e., what values the instances ruled out by the constraints are mapped to. 

Since small decision trees are often considered interpretable, we chose them for our experiments, as they also allow one to ``read-off'' reasons from their branches in order to compare with our most-parsimonious reasons. In fact, using decision trees yielded the following observations.
In Case Study 1, we showed that while being indeed interpretable, the decision rules derived from decision trees may not be the most parsimonious explanations, and demonstrated this point with an example in the discussion of Case Study 1.
This left us with a question: are decision rules derived from an \emph{optimal} decision tree, sufficient reasons? A question we leave for future investigations.

Another avenue for future work is to  experiment with models that are considered to be less interpretable, such as Bayesian Networks, Random Forests or Neural Networks. In these cases one might compare our most-parsimonious reasons with local reasons produced by the machine-learning community, such as ``anchors''~\cite{Ribeiro:AAAI18}, which would amount to finding approximate reasons \cite{ignatiev2019validating}, clearly out of the scope of this paper.

Finally, our work opens up applications that are currently only available in the unconstrained setting, e.g., computing complete reasons, decision counterfactuals, and classifier bias~\cite{darwiche2020reasons}.

\bibliographystyle{named}
\bibliography{main}

\end{document}